\documentclass{article}

\usepackage[final, nonatbib]{neurips_2018}

\usepackage[utf8]{inputenc} 
\usepackage[T1]{fontenc}    
\usepackage{hyperref}       
\usepackage{url}            
\usepackage{booktabs}       
\usepackage{amsfonts}       
\usepackage{nicefrac}       
\usepackage{microtype}      

\usepackage[normalem]{ulem}  

\newcommand{\feat}{\ensuremath{\boldsymbol\phi}}
\newcommand{\w}{\ensuremath{\mathbf{w}}}
\newcommand{\wopt}{\ensuremath{\mathbf{w}^*}}

\usepackage{algorithm}
\usepackage{float}
\usepackage{times}
\usepackage[]{hyperref}
\usepackage{mathtools}
\usepackage{multirow,color,graphicx}
\usepackage{subcaption}
\usepackage{xfrac,amsfonts,amsbsy}
\usepackage[titletoc]{appendix}
\usepackage{xspace}
\usepackage{savesym,verbatim}
\usepackage[titletoc]{appendix}
\usepackage{paralist}
\usepackage{tikz}
\usepackage{verbatim}
\usepackage{amsthm}
\usepackage{algcompatible}
\usepackage{algpseudocode}
\usepackage{enumitem}
\usepackage{amssymb}

\newtheorem{theorem}{Theorem}
\newtheorem{proposition}{Proposition}

\makeatletter

\newcommand{\Rmnum}[1]{\expandafter\@slowromancap\romannumeral #1@}
\makeatother

\DeclareMathOperator*{\argmin}{arg\,min}
\DeclareMathOperator*{\argmax}{arg\,max}

\newcommand{\R}{\mathbb{R}}
\newcommand{\E}{\mathbb{E}}

\newcommand{\mdp}{\ensuremath{\mathcal{M}}}

\def \argmax {\mathop{\rm arg\,max}}
\def \argmin {\mathop{\rm arg\,min}}

\newcommand{\cF}{{\mathcal{F}}}

\DeclareMathOperator\diam{diam}
\DeclareMathOperator\pr{pr}

\title{Teaching Inverse Reinforcement Learners\\ via Features and Demonstrations}
\author{
    Luis Haug\\
    Department of Computer Science \\
    ETH Zurich\\
    \texttt{lhaug@inf.ethz.ch}
    \And
    Sebastian Tschiatschek\\
    Microsoft Research\\
    Cambridge, UK\\    
    \texttt{setschia@microsoft.com}
    \And
    Adish Singla\\
    Max Planck Institute for Software Systems\\
    Saarbr\"ucken, Germany\\
    \texttt{adishs@mpi-sws.org}
}
\begin{document}
\maketitle

\begin{abstract}
    
    Learning near-optimal behaviour from an expert's demonstrations
    typically relies on the assumption that the learner knows the
    features that the true reward function depends on. In this paper,
    we study the problem of learning from demonstrations in the
    setting where this is \emph{not} the case, i.e., where there is a
    mismatch between the worldviews of the learner and the expert. We
    introduce a natural quantity, the \emph{teaching risk}, which
    measures the potential suboptimality of policies that look optimal
    to the learner in this setting. We show that bounds on the
    teaching risk guarantee that the learner is able to find a
    near-optimal policy using standard algorithms based on inverse
    reinforcement learning. Based on these findings, we suggest a
    teaching scheme in which the expert can decrease the teaching risk
    by updating the learner's worldview, and thus ultimately enable
    her to find a near-optimal policy.
\end{abstract}


\section{Introduction}\label{sec:intro}
Reinforcement learning has recently led to impressive and widely
recognized results in several challenging application domains,
including game-play, e.g., of classical games (Go) and Atari games.
In these applications, a clearly defined reward function, i.e.,
whether a game is won or lost in the case of Go or the number of
achieved points in the case of Atari games, is optimized by a
reinforcement learning agent interacting with the environment.

However, in many applications it is very difficult to specify a reward
function that captures all important aspects.  For instance, in an
autonomous driving application, the reward function of an autonomous
vehicle should capture many different desiderata, including the time
to reach a specified goal, safe driving characteristics, etc. In such
situations, learning from demonstrations can be a remedy, transforming
the need of specifying the reward function to the task of providing an
expert's demonstrations of desired behaviour; we will refer to this
expert as the \emph{teacher}. Based on these demonstrations, a
learning agent, or simply \emph{learner} attempts to infer a
(stochastic) policy that approximates the feature counts of the
teacher's demonstrations. Examples for algorithms that can be used to
that end are those in \cite{Abbeel-Ng--ICML04} and \cite{ziebart2008maximum},
which use inverse reinforcement learning (IRL) to estimate a reward
function for which the demonstrated behaviour is optimal, and then
derive a policy based on that.

For this strategy to be successful, i.e., for the learner to find a
policy that achieves good performance with respect to the reward
function set by the teacher, the learner has to know what features the
teacher considers and the reward function depends on. However, as we
argue, this assumption does not hold in many real-world
applications. For instance, in the autonomous driving application, the
teacher, e.g., a human driver, might consider very different features,
including high-level semantic features, while the learner, i.e., the
autonomous car, only has sensory inputs in the form of distance
sensors and cameras providing low-level semantic features.  In such a
case, there is a mismatch between the teacher's and the learner's
features which can lead to degraded performance and unexpected
behaviour of the learner.

In this paper we investigate exactly this setting. We assume that the
true reward function is a linear combination of a set of features
known to the teacher. The learner also assumes that the reward
function is linear, but in features which are different from the truly
relevant ones; e.g., the learner could only observe a subset of those
features.  In this setting, we study the potential decrease in
performance of the learner as a function of the learner's worldview.
We introduce a natural and easily computable quantity, the
\emph{teaching risk}, which bounds the maximum possible performance
gap of the teacher and the learner.

We continue our investigation by considering a teaching scenario in
which the teacher can provide additional features to the learner,
e.g., add additional sensors to an autonomous vehicle.  This naturally
raises the question which features should be provided to the learner
to maximize her performance. To this end, we propose an algorithm
that greedily minimizes the teaching risk, thereby shrinking the
maximal gap in performance that policies optimized with respect to the
learner's resp.\ teacher's worldview can have.


Or main contributions are:
\begin{enumerate}
\item We formalize the problem of worldview mismatch for reward
    computation and policy optimization based on demonstrations.
\item We introduce the concept of \emph{teaching risk}, bounding the
    maximal performance gap of the teacher and the learner as a
    function of the learner's worldview and the true reward function.
\item We formally analyze the teaching risk and its properties, giving
    rise to an algorithm for teaching a learner with an
    incomplete worldview.
\item We substantiate our findings in a large set of
    experiments.
\end{enumerate}




\section{Related Work}\label{sec:relatedwork}
Our work is related to the area of algorithmic machine teaching, where
the objective is to design effective teaching algorithms to improve
the learning process of a
learner~\cite{zhu2018overview,zhu2015machine}. Machine teaching has
recently been studied in the context of diverse real-world
applications such as personalized education and intelligent tutoring
systems~\cite{hunziker2018teachingmultiple,rafferty2016faster,patil2014optimal},
social robotics \cite{cakmak2014eliciting}, adversarial machine
learning~\cite{mei2015using}, program
synthesis~\cite{mayer2017proactive}, and human-in-the-loop
crowdsourcing systems~\cite{singla2014near,singla2013actively}.
However, different from ours, most of the current work in machine
teaching is limited to supervised learning settings, and to a setting
where the teacher has full knowledge about the learner's model.

Going beyond supervised learning,
\cite{cakmak2012algorithmic,danielbrown2018irl,kamalaruban2018assisted} have studied the
problem of teaching an IRL agent, similar in spirit to what we do in
our work. Our work differs from their work in several
aspects---they assume that
the teacher has full knowledge of the learner's feature space, and
then provides a near-optimal set/sequence of demonstrations;
we consider a more realistic setting where there is a mismatch between
the teacher's and the learner's feature space. Furthermore, in our
setting, the teaching signal is a mixture of demonstrations and
features.

Our work is also related to teaching via explanations and features as
explored recently by \cite{macaodha18teaching} in a supervised
learning setting. However, we explore the space of teaching by
explanations when teaching an IRL agent, which makes it technically
very different from \cite{macaodha18teaching}. Another important
aspect of our teaching algorithm is that it is adaptive in nature, in
the sense that the next teaching signal accounts for the current
performance of the learner (i.e., worldview in our setting). Recent
work of \cite{chen2018understanding,liu2017iterative,yeo2019iterative} have studied
adaptive teaching algorithms, however only in a supervised learning
setting.

Apart from machine teaching, our work is related to
\cite{stadie2017-3rd-person} and \cite{sermanet2017-time-contrastive},
which also study imitation learning problems in which the teacher and
the learner view the world differently. However, these two works are
technically very different from ours, as we consider the problem of
providing teaching signals under worldview mismatch from the
perspective of the teacher.


\section{The Model}
\label{sec:model}


\paragraph{Basic definitions.} Our environment is described by a
\emph{Markov decision process} $\mdp = (S, A, T, D, R, \gamma)$, where
$S$ is a finite set of states, $A$ is a finite set of available
actions, $T$ is a family of distributions on $S$ indexed by
$S \times A$ with $T_{s, a}(s')$ describing the probability of
transitioning from state $s$ to state $s'$ when action $a$ is taken,
$D$ is the initial-state distribution on $S$ describing the
probability of starting in a given state, $R\colon S \rightarrow \R$ is a
reward function and $\gamma \in (0,1)$ is a discount factor. We assume
that there exists a feature map $\feat\colon S \to \R^k$ such that the
reward function is linear in the features given by $\feat$, i.e.,
\begin{equation*}
    R(s) = \langle \wopt, \feat(s) \rangle
\end{equation*}
for some $\wopt \in \R^k$ which we assume to satisfy $\Vert \wopt
\Vert = 1$. 

By a \emph{policy} we mean a family of distributions on $A$ indexed by
$S$, where $\pi_s(a)$ describes the probability of taking action $a$ in
state $s$. We denote by $\Pi$ the set of all such policies. The
performance measure for policies we are interested in is the
\emph{expected discounted reward}
$R(\pi) := \E \left(\sum_{t=0}^\infty \gamma^t R(s_t)\right)$,
where the expectation is taken with respect to the distribution over
trajectories $(s_0, s_1, \dots)$ induced by $\pi$ together with the
transition probabilities $T$ and the initial-state distribution
$D$. We call a policy $\pi$ \emph{optimal} for the reward function
$R$ if $\pi \in \argmax_{\pi' \in \Pi} R(\pi')$. Note that
\begin{equation*}
    R(\pi) = \langle \wopt, \mu(\pi) \rangle,
\end{equation*}
where $\mu\colon \Pi \to \R^k$,
$\pi \mapsto \E \left( \sum_{t=0}^\infty \gamma^t \feat(s_t)\right)$,
is the map taking a policy to its vector of \emph{(discounted) feature
    expectations}. Note also that the image $\mu(\Pi)$ of this map is
a bounded subset of $\R^k$ due to the finiteness of $S$ and the
presence of the discounting factor $\gamma \in (0,1)$; we denote by
$\diam \mu(\Pi) = \sup_{\mu_0, \mu_1 \in \mu(\Pi)} \Vert \mu_0 - \mu_1
\Vert$ its diameter. Here and in what follows, $\Vert \cdot \Vert$
denotes the Euclidean norm.


\paragraph{Problem formulation.} We consider a learner $L$
and a teacher $T$, whose ultimate objective is that $L$ finds a
near-optimal policy $\pi^L$ with the help of
$T$. 

The challenge we address in this paper is that of achieving this
objective under the assumption that there is a \emph{mismatch between
    the worldviews of $L$ and $T$}, by which we mean the following:
Instead of the ``true'' feature vectors $\feat(s)$, $L$ observes
feature vectors $A^L \feat(s) \in \R^{\ell}$, where
\begin{equation*}
    A^L\colon \R^k \to \R^\ell
\end{equation*}
is a linear map (i.e., a matrix) that we interpret as $L$'s
worldview. The simplest case is that $A^L$ selects a subset of the
features given by $\feat(s)$, thus modelling the situation where $L$
only has access to a subset of the features relevant for the true
reward, which is a reasonable assumption for many real-world
situations. More generally, $A^L$ could encode different weightings of
those features. 

The question we ask is whether and how $T$ can provide demonstrations
or perform other teaching interventions, in a way such as to make sure
that $L$ achieves the goal of finding a policy with near-optimal
performance. 

\paragraph{Assumptions on the teacher and on the learner.}
We assume that $T$ knows the full specification of the MDP as well as
$L$'s worldview $A^L$, and that she can help $L$ to learn in two
different ways:
\begin{enumerate}
\item By providing $L$ with demonstrations of behaviour in the MDP;
\item By updating $L$'s worldview $A^L$.
\end{enumerate}

Demonstrations can be provided in the form of trajectories sampled
from a (not necessarily optimal) policy $\pi^T$, or in the form of
(discounted) feature expectations of such a policy. The method by which $T$ can
update $A^L$ will be discussed in Section \ref{sec.algorithms}. Based
on $T$'s instructions, $L$ then attemps to train a policy $\pi^L$
whose feature expectations approximate those of $\pi^T$. Note that, if
this is successful, the performance of $\pi^L$ is close to that of
$\pi^T$ due to the form of the reward function.

We assume that $L$ has access to an algorithm that enables her to do
the following: Whenever she is given sufficiently many demonstrations
sampled from a policy $\pi^T$, she is able to find a policy $\pi^L$
whose feature expectations in \emph{her} worldview approximate those
of $\pi^T$, i.e., $A^L \mu(\pi^L) \approx A^L \mu(\pi^T)$. Examples
for algorithms that $L$ could use to that end are the algorithms in
\cite{Abbeel-Ng--ICML04} and \cite{ziebart2008maximum} which are based
on IRL. The following discussion does not require any further
specification of what precise algorithm $L$ uses in order to match
feature expectations.


\begin{figure}[t]
    \begin{subfigure}[t]{0.49\textwidth}
        \includegraphics[width=1\textwidth]{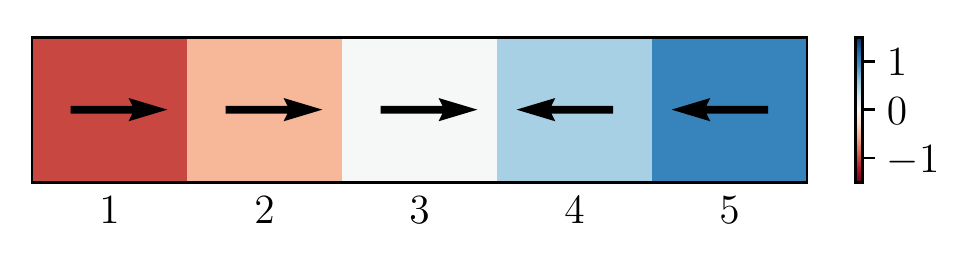}
        \subcaption{$\pi_0$}
        \label{fig:1d-policies-policy0}
    \end{subfigure}
    \begin{subfigure}[t]{0.49\textwidth}
        \includegraphics[width=1\textwidth]{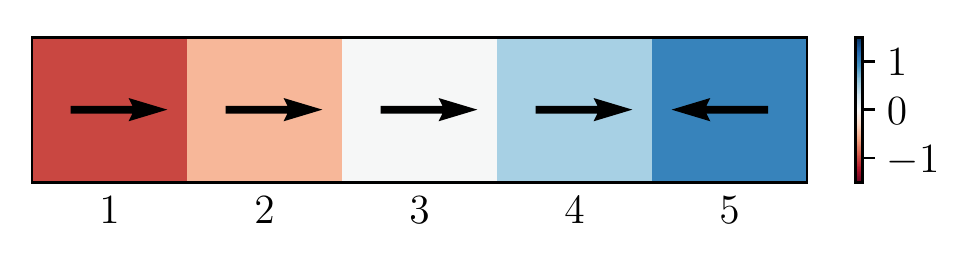}
        \subcaption{$\pi^*$}
        \label{fig:1d-policies-optimal-policy}
    \end{subfigure}
    \caption{A simple example to illustrate the challenges arising
        when teaching under worldview mismatch. We consider an MDP in
        which $S = \{s_1, \dots, s_5\}$ is the set of cells in the
        gridworld displayed, $A = \{\leftarrow, \rightarrow\}$, and
        $R(s) = \langle \wopt, \feat(s) \rangle$ with feature map
        $\feat: S \to \R^{5}$ taking $s_i$ to the one-hot vector
        $e_i \in \R^{5}$. The initial state distribution is uniform
        and the transition dynamics are deterministic. More
        specifically, when the agent takes action $\rightarrow$
        (resp. $\leftarrow$), it moves to the neighboring cell to the
        right (resp. left); when the agent is in the rightmost
        (resp. leftmost) cell, the action $\rightarrow$
        (resp. $\leftarrow$) is not permitted. The reward weights are
        given by $\wopt = (-1, -0.5, 0, 0.5, 1)^T \in \R^5$ up to
        normalization; the values
        $R(s_i) = \langle \wopt, \feat(s_i)\rangle$ are also encoded
        by the colors of the cells. The policy $\pi^*$ in
        (\subref{fig:1d-policies-optimal-policy}) is the optimal
        policy with respect to the true reward function. Assuming that
        the learner $L$ only observes the feature corresponding to the
        central cell, i.e.,
        $A^L = (0 \quad 0 \quad 1 \quad 0 \quad 0) \in \R^{1 \times
            5}$, the policy $\pi_0$ in
        (\subref{fig:1d-policies-policy0}) is a better teaching policy
        in the worst-case sense. See the main text for a detailed
        description.}
    \label{fig:1d-policies}
\end{figure}

\paragraph{Challenges when teaching under worldview mismatch.} 
If there was no mismatch in the worldview (i.e., if $A^L$ was the
identity matrix in $\R^{k \times k}$), then the teacher could simply
provide demonstrations from the optimal policy $\pi^*$ to achieve the
desired objective. However, the example in Figure~\ref{fig:1d-policies}
illustrates that this is not the case when there is a mismatch between
the worldviews.

For the MDP in Figure~\ref{fig:1d-policies}, assume that the teacher
provides demonstrations using $\pi^T = \pi^*$, which moves to the
rightmost cell as quickly as possible and then alternates between
cells 4 and 5 (see Figure~\ref{fig:1d-policies-optimal-policy}). Note
that the policy $\widetilde \pi^*$ which moves to the leftmost cell as
quickly as possible and then alternates between cells 1 and 2, has the
same feature expectations as $\pi^*$ in the learner's worldview; in
fact, $\widetilde \pi^*$ is the unique policy other than $\pi^*$ with
that property (provided we restrict to deterministic policies). As the
teacher is unaware of the internal workings of the learner, she has no
control over which of these two policies the learner will eventually
learn by matching feature expectations.

However, the teacher can ensure that the learner achieves better
performance in a worst case sense by providing demonstrations tailored
to the problem at hand. In particular, assume that the teacher uses
$\pi^T = \pi_0$, the policy shown in
Figure~\ref{fig:1d-policies-policy0}, which moves to the central cell
as quickly as possible and then alternates between cells 3 and 4. The
only other policy $\widetilde \pi_0$ with which the learner could
match the feature expectations of $\pi_0$ in her worldview
(restricting again to deterministic policies) is the one that moves to
the central cell as quickly as possible and then alternates between
states 2 and 3.

Note that
$R(\pi^*) > R(\pi_0) > R(\widetilde \pi_0) > R(\widetilde \pi^*)$, and
hence $\pi_0$ is a better teaching policy than $\pi^*$ regarding the
performance that a learner matching feature expectations in her
worldview achieves in the worst case. In particular, this example
shows that providing demonstrations from the truly optimal policy
$\pi^*$ does not guarantee that the learner's policy achieves good
performance in general.

\section{Teaching Risk}
\label{sec:teaching-risk}

\paragraph{Definition of teaching risk.} 
The fundamental problem in the setting described in Section \ref{sec:model} is that two policies
$\pi_0, \pi_1$ that perform equally well with respect to any estimate
$s \mapsto \langle \w^L, A^L \feat(s) \rangle$ that $L$ may have of
the reward function, may perform very differently with respect to the
true reward function. Hence, even if $L$ is able to imitate the
behaviour of the teacher well in \emph{her} worldview, there is
genenerally no guarantee on how good her performance is with respect
to the true reward function. For an illustration, see Figure
\ref{fig:1d-policies-teaching-risk}.


To address this problem, we define the following quantity: The
\textbf{teaching risk} for a given worldview $A^L$ with respect to
reward weights $\wopt$ is
\begin{equation}
    \label{defn:teaching-risk}
    \rho(A^L; \wopt) := \max_{v \in \ker A^L, \Vert v \Vert \leq 1}
    \langle \wopt, v \rangle.
\end{equation}
Here $\ker A^L = \{v \in \R^k ~|~A^L v = 0\}$ and
$\ker \wopt = \{v \in \R^k ~|~ \langle \wopt, v \rangle = 0\}$ denote
the kernels of $A^L$ resp.\ $\langle \wopt, \cdot
\rangle$. Geometrically, $\rho(A^L; \wopt)$ is the cosine of the angle
between $\ker A^L$ and $\wopt$; in other words, $\rho(A^L; \wopt)$
measures the degree to which $\ker A^L$ deviates from satisfying
$\ker A^L \subseteq \ker \wopt$. Yet another way of characterizing the
teaching risk is as $\rho(A^L; \wopt) = \Vert \pr(\wopt) \Vert$, where
$\pr: \R^k \to \ker A^L$ denotes the orthogonal projection onto
$\ker A^L$.

\paragraph{Significance of teaching risk.}
To understand the significance of the teaching risk in our context,
assume that $L$ is able to find a policy $\pi^L$ which matches the
feature expectations of $T$'s (not necessarily optimal) policy $\pi^T$
perfectly in her worldview, which is equivalent to
$\mu(\pi^T) - \mu(\pi^L) \in \ker A^L$. Directly from the definition
of the teaching risk, we see that the gap between their performances
with respect to the \emph{true} reward function satisfies
\begin{align}
  \label{eq:tr-basic-estimate}
  \begin{split}
      \vert \langle \wopt,
      \mu(\pi^T) - \mu(\pi^L)
      \rangle \vert
      \leq \rho(A^L; \wopt) \cdot \Vert \mu(\pi^T) - \mu(\pi^L) \Vert,
  \end{split}
\end{align}
with equality if $\mu(\pi^T) - \mu(\pi^L)$ is proportional to a vector
$v$ realizing the maximum in \eqref{defn:teaching-risk}. If the
teaching risk is large, this performance gap can generally be large as
well. This motivates the interpretation of $\rho(A^L; \wopt)$ as a
measure of the risk when teaching the task modelled by an MDP with
reward weights $\wopt$ to a learner whose worldview is represented by
$A^L$.

\begin{figure}[t]
    \begin{subfigure}[]{0.49\textwidth}
        \includegraphics[width=1\textwidth]{fig/policy0.pdf}
        \subcaption{$\pi_0$}
        \label{fig:1d-policies-tr-policy0}
    \end{subfigure}
    \begin{subfigure}[]{0.49\textwidth}
        \includegraphics[width=1\textwidth]{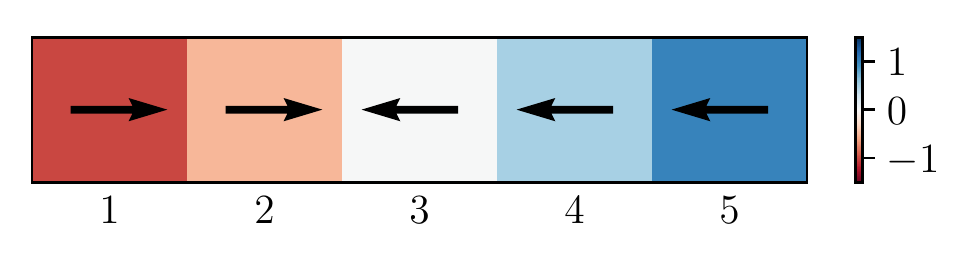}
        \subcaption{$\pi_1$}
        \label{fig:1d-policies-tr-policy1}
    \end{subfigure}
    \caption{Two policies in the environment introduced in Figure
        \ref{fig:1d-policies} (the policy $\pi_0$ here is identical to
        the one in Figure
        \ref{fig:1d-policies}(\subref{fig:1d-policies-policy0})). We
        assume again that $L$ can only observe the feature
        corresponding to the central cell. Provided that the initial
        state distribution is uniform, the feature expectations of
        $\pi_0$ and $\pi_1$ in $L$'s worldview are equal, and hence
        these policies perform equally well with respect to any estimate of the
        reward function that $L$  may have. In fact, both look optimal
        to $L$ if she assumes that the central cell carries positive
        reward. However, their performance with respect to the true
        reward function is positive for $\pi_0$ but negative for
        $\pi_1$. This illustrates that, if all we know about $L$ is
        that she matches feature counts in her worldview, we can
        generally not give good performance guarantees for the policy
        she finds.}
    \label{fig:1d-policies-teaching-risk}
\end{figure}


On the other hand, smallness of the teaching risk implies that this
performance gap cannot be too large. 
The following theorem, proven in the Appendix, 
generalizes the bound in \eqref{eq:tr-basic-estimate} to the situation in which
$\pi^L$ only approximates the feature expectations of $\pi^T$.

\vspace{4mm}
\begin{theorem}
    \label{thm:guarantee-bounded-risk}
    Assume that
    $ \Vert A^L (\mu(\pi^T) - \mu(\pi^L)) \Vert < \varepsilon$.  Then
    the gap between the true performances of $\pi^L$ and $\pi^T$
    satisfies
    \begin{equation*}
        \label{ineq:guarantee-bounded-risk}
        \vert \langle \wopt, \mu(\pi^T) - \mu(\pi^L) \rangle \vert < \frac{\varepsilon}{\sigma(A^L)} +
        \rho(A^L; \wopt) \cdot \diam \mu(\Pi)
    \end{equation*}
    with $\sigma(A^L) = \min_{v \perp \ker A^L, \Vert v \Vert = 1} \Vert
    A^L v \Vert$.
\end{theorem}

Theorem \ref{thm:guarantee-bounded-risk} shows the following: If $L$
imitates $T$'s behaviour well in her worldview (meaning that
$\varepsilon$ can be chosen small) and if the teaching risk
$\rho(A^L; \wopt)$ is sufficiently small, then $L$ will perform nearly
as well as $T$ with respect to the true reward. In particular, if
$T$'s policy is optimal, $\pi^T = \pi^*$, then $L$'s policy $\pi^L$ is
guaranteed to be near-optimal.

The quantity $\sigma(A^L)$ appearing in Theorem
\ref{thm:guarantee-bounded-risk} is a bound on the amount to which
$A^L$ distorts lengths of vectors in the orthogonal complement of
$\ker A^L$. Note that $\sigma(A^L)$ is independent of the teaching
risk, in the sense that one can change it, e.g., by rescaling
$A^L \to \alpha A^L$ by some $\alpha \in \R_+$, without changing the
teaching risk.

\paragraph{Teaching risk as obstruction to recognizing optimality.}



We now provide a second motivation for the consideration of the
teaching risk, by interpreting it as a quantity that measures the
degree to which truly optimal policies deviate from looking optimal to
$L$. We make the technical assumption that $\mu(\Pi)$ is the closure
of a bounded open set with smooth boundary $\partial \mu(\Pi)$ (this
will only be needed for the proofs). Our first observation is the
following:


\vspace{4mm}
\begin{proposition}
    \label{prop:teaching-risk-suboptimality} Let $\pi^*$ be a policy
    which is optimal for $s \mapsto \langle \wopt, \feat(s) \rangle$.
    If $\rho(A^L; \wopt) > 0$, then $\pi^*$ is suboptimal with respect
    to \emph{any} choice of reward function
    $s \mapsto \langle \w, A^L\feat(s)\rangle$ with
    $\w \in \R^{\ell}$.
\end{proposition}
In view of Proposition
\ref{prop:teaching-risk-suboptimality}, a natural question is whether
we can bound the suboptimality, in $L$'s view, of a truly optimal
policy in terms of the teaching risk. The following theorem provides
such a bound:

\vspace{4mm}
\begin{theorem}
    \label{thm:guarantee-bounded-suboptimality}
    Let $\pi^*$ be a policy which is optimal for
    $s \mapsto \langle \wopt, \feat(s) \rangle$. There exists a unit
    vector $\wopt_L \in \R^\ell$ such that
    \begin{equation*}
        \label{ineq:guarantee-bounded-suboptimality}
        \max_{\mu \in \mu(\Pi)} \left \langle \wopt_L, A^L \mu \right
        \rangle - \left \langle
            \wopt_L , A^L \mu(\pi^*) \right \rangle \leq 
        \frac{\diam \mu(\Pi) \cdot \Vert A^L \Vert \cdot \rho(A^L; \wopt)}{\sqrt{1 - \rho(A^L; \wopt)^2}},
    \end{equation*}
    where $\Vert A^L \Vert = \max_{v \in \R^k, \Vert v \Vert = 1}
    \Vert A^L v \Vert$.
\end{theorem}

Proofs of Proposition~\ref{prop:teaching-risk-suboptimality} and Theorem~\ref{ineq:guarantee-bounded-suboptimality} are given in the Appendix. 
Note that the expression on the right hand side of
the inequality in Theorem \ref{thm:guarantee-bounded-suboptimality}
tends to $0$ as $\rho(A^L; \wopt) \to 0$, provided $\Vert A^L \Vert$
is bounded. Theorem \ref{thm:guarantee-bounded-suboptimality}
therefore implies that, if $\rho(A^L; \wopt)$ is small, a truly
optimal policy $\pi^*$ is near-optimal for \emph{some} choice of
reward function linear in the features $L$ observes, namely, the
reward function $s \mapsto \langle \wopt_L, \feat(s) \rangle$ with
$\wopt_L \in \R^\ell$ the vector whose existence is claimed by the
theorem.




\section{Teaching}
\label{sec.algorithms}

\begin{algorithm}[t]
    \caption{\textsc{TRGreedy}: Feature- and demo-based
        teaching with TR-greedy feature selection}
    \label{algo:feature-and-demo-based-teaching}
    \begin{algorithmic}
        \Require Reward vector $\w^*$, set of teachable features
        $\cF$, feature budget $B$, initial worldview $A^L$, teacher
        policy $\pi^T$, initial learner policy $\pi^L$, performance
        threshold $\varepsilon$.
        
        \For{$i = 1,\dots, B$}

        \If{
            $\vert \langle \wopt, \mu(\pi^L) \rangle - \langle
            \wopt, \mu(\pi^T) \rangle \vert > \varepsilon$}
        \State
        $f \leftarrow \argmin_{f \in \cF} \rho(A^L \oplus
        \langle f, \cdot \rangle; \w^*)$ \Comment{$T$ selects
            feature to teach}

        \State
        $A^L \leftarrow A^L \oplus \langle f, \cdot \rangle$
        \Comment{$L$'s worldview gets updated}

        \State
        $\pi^L \leftarrow \textsc{Learning}(\pi^L,
        A^L\mu(\pi^T))$ \Comment{$L$ trains a new policy}%
        \Else \State \Return $\pi^L$%
    
        \EndIf{}%
        \EndFor{} \State \Return $\pi^L$
    \end{algorithmic}
\end{algorithm}

\paragraph{Feature teaching.} The discussion in the last section shows
that, under our assumptions on how $L$ learns, a teaching scheme in
which $T$ solely provides demonstrations to $L$ can generally, i.e.,
without any assumption on the teaching risk, not lead to reasonable
guarantees on the learner's performance with respect to the true
reward. A natural strategy is to introduce additional teaching
operations by which the teacher can update $L$'s worldview $A^L$ and
thereby decrease the teaching risk.

The simplest way by which the teacher $T$ can change $L$'s worldview is by
informing her about features $f \in \R^k$ that are relevant to
performing well in the task, thus causing her to update her worldview
$A^L\colon \R^k \to \R^{\ell}$ to
\begin{equation*}
    A^L \oplus \langle f, \cdot \rangle \colon \R^k \to \R^{\ell + 1}.
\end{equation*}
Viewing $A^L$ as a matrix, this operation appends $f$ as a row to
$A^L$. (Strictly speaking, the feature that is thus provided is
$s \mapsto \langle f, \feat(s) \rangle$; we identify this map with the
vector $f$ in the following and thus keep calling $f$ a ``feature''.)

This operation has simple interpretations in the settings we are
interested in: If $L$ is a human learner, ``teaching a feature'' could
mean making $L$ aware that a certain quantity, which she might not
have taken into account so far, is crucial to achieving high
performance. If $L$ is a machine, such as an autonomous car or a
robot, it could mean installing an additional sensor.


\paragraph{Teachable features.} Note that if $T$ could provide
arbitrary vectors $f \in \R^k$ as new features, she could always, no
matter what $A^L$ is, decrease the teaching risk to zero in a single
teaching step by choosing $f = \wopt$, which amounts to telling $L$
the true reward function. We assume that this is not possible, and
that instead only the elements of a fixed finite set of
\emph{teachable features}
\begin{equation*}
    \cF = \{f_i ~|~ i \in I\} \subset \R^k
\end{equation*}
can be taught. In real-world applications, such constraints could come
from the limited availability of sensors and their costs; in the case
that $L$ is a human, they could reflect the requirement
that features need to be interpretable, i.e., that they can only be
simple combinations of basic observable quantities.


\paragraph{Greedy minimization of teaching risk.}
Our basic teaching algorithm \textsc{TRGreedy} (Algorithm
\ref{algo:feature-and-demo-based-teaching}) works as follows: $T$ and
$L$ interact in rounds, in each of which $T$ provides $L$ with the
feature $f \in \cF$ which reduces the teaching risk of $L$'s worldview
with respect to $\wopt$ by the largest amount. $L$ then trains a
policy $\pi^L$ with the goal of imitating her current view
$A^L \mu(\pi^T)$ of the feature expectations of the teacher's policy;
the \textsc{Learning} algorithm she uses could be the apprenticeship
learning algorithm from \cite{Abbeel-Ng--ICML04}.

\paragraph{Computation of the teaching risk.} The computation of the
teaching risk required of $T$ in every round of Algorithm
\ref{algo:feature-and-demo-based-teaching} can be performed as
follows: One first computes the orthogonal complement of
$\ker A^L \cap \ker \wopt$ in $\R^k$ and intersects that with
$\ker A^L$, thus obtaining (generically) a 1-dimensional subspace
$\lambda = (\ker A^L \cap \ker \wopt)^\perp \cap \ker A^L$ of $\R^k$;
this can be done using SVD. The teaching risk is then
$\rho(A^L; \wopt) = \langle\wopt, v^\perp\rangle$ with $v^\perp$ the
unique unit vector in $\lambda$ with
$\langle \wopt, v^\perp \rangle > 0$.



\section{Experiments}\label{sec:experiments}

\begin{figure}[t]    
    \begin{subfigure}[t]{0.329\textwidth}
        \includegraphics[width=\textwidth]{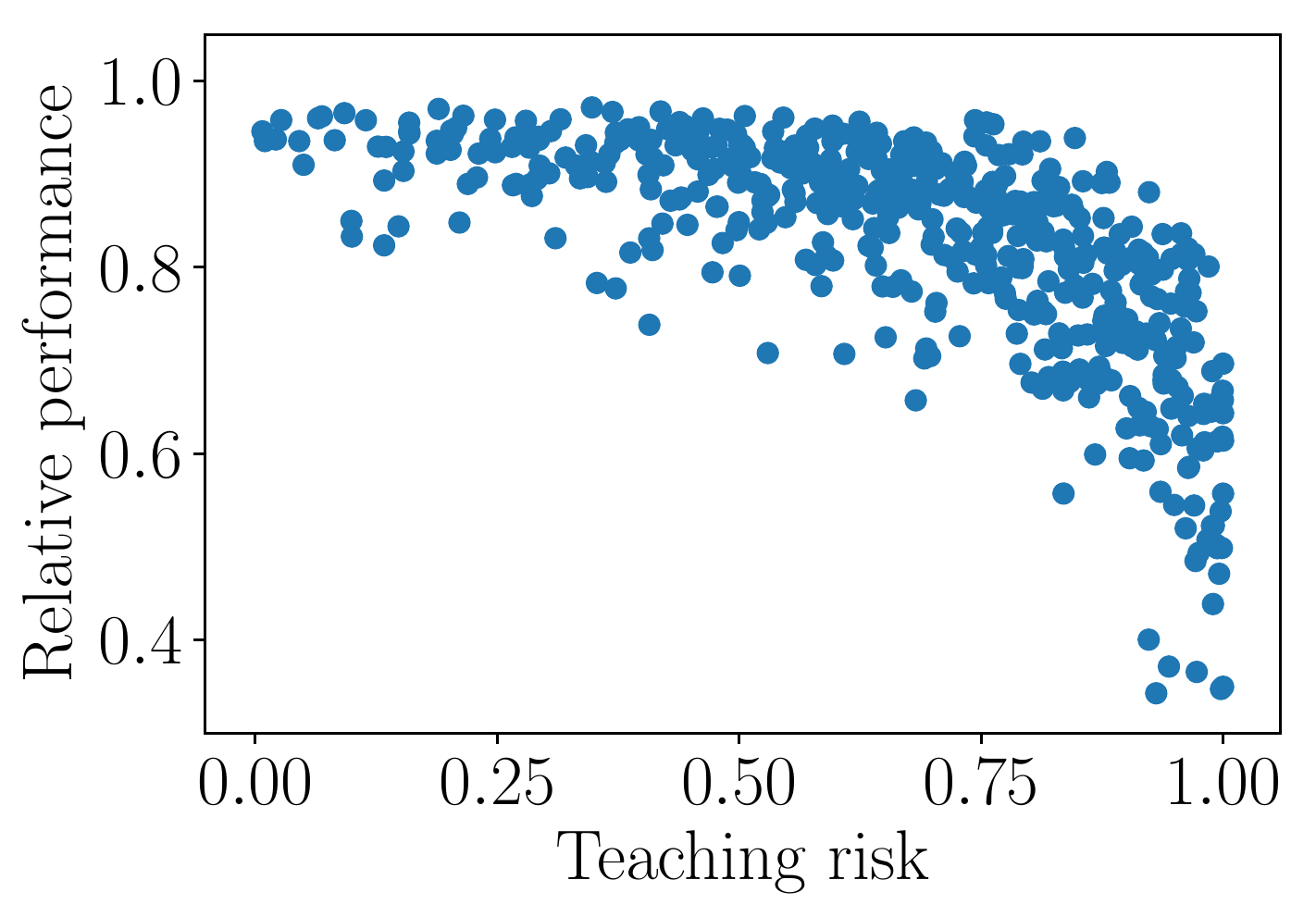}
        \subcaption{$\gamma = 0.5$}
        \label{fig:tr-vs-perf.a}
    \end{subfigure}
    \begin{subfigure}[t]{0.329\textwidth}
        \includegraphics[width=\textwidth]{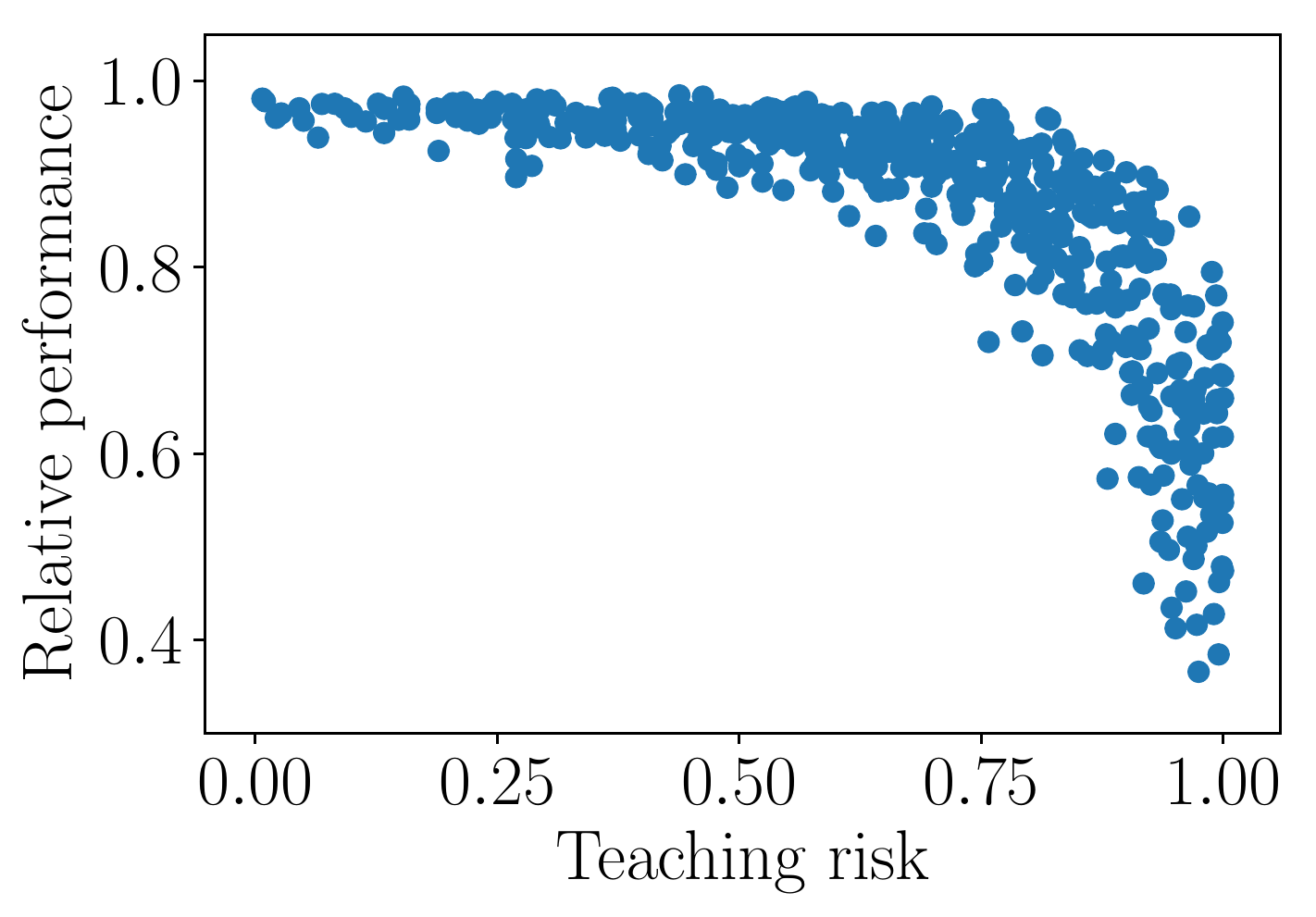}
        \subcaption{$\gamma = 0.75$}
        \label{fig:tr-vs-perf.b}
    \end{subfigure}
    \begin{subfigure}[t]{0.329\textwidth}
        \includegraphics[width=\textwidth]{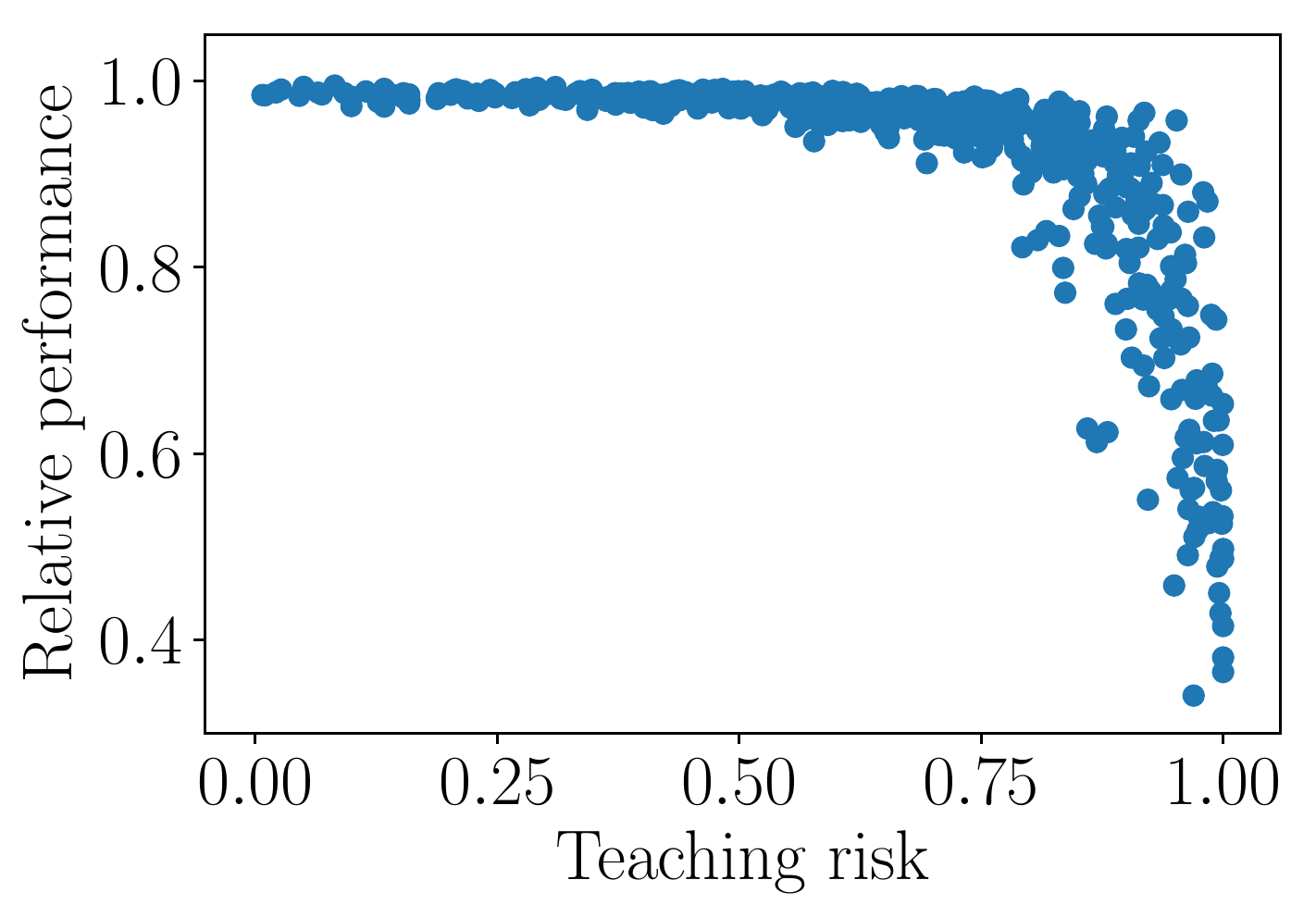}
        \subcaption{$\gamma = 0.9$}
        \label{fig:tr-vs-perf.c}
    \end{subfigure}
    \caption{Performance vs.\ teaching risk. Each point in the plots
        shows the relative performance that a learner $L$ with a
        random worldview matrix $A^L$ achieved after one round of
        learning and the teaching risk of $A^L$. For all plots, a
        gridworld with $N=20$, $n=2$ was used. The reward vector
        $\wopt$ was sampled randomly in each round.
        (\subref{fig:tr-vs-perf.a})--(\subref{fig:tr-vs-perf.c})
        correspond to different values of the discount factor
        $\gamma$.}
    \label{fig:tr-vs-perf}
\end{figure}

Our experimental setup is similar to the one in
\cite{Abbeel-Ng--ICML04}, i.e., we use $N \times N$ gridworlds in
which non-overlapping square regions of neighbouring cells are grouped
together to form $n \times n$ macrocells for some $n$ dividing
$N$. The state set $S$ is the set of gridpoints, the action set is
$A = \{\leftarrow, \rightarrow, \uparrow, \downarrow\}$, and the
feature map $\feat\colon S \to \R^k$ maps a gridpoint belonging to
macrocell $i \in \{1, \dots, (N/n)^2\}$ to the one-hot vector
$e_i \in \R^k$; the dimension of the ``true'' feature space is
therefore $k = (N/n)^2$. Note that these gridworlds satisfy the quite
special property that for states $s \neq s'$, we either have
$\feat(s) = \feat(s')$ (if $s, s'$ belong to the same macrocell), or
$\feat(s) \perp \feat(s')$. The reward weights $\wopt \in \R^k$ are
sampled randomly for all experiments unless mentioned otherwise. As
the \textsc{Learning} algorithm within
Algorithm~\ref{algo:feature-and-demo-based-teaching}, we use the
\emph{projection version} of the apprenticeship learning algorithm
from \cite{Abbeel-Ng--ICML04}.

\paragraph{Performance vs.\ teaching risk.}
The plots in Figure \ref{fig:tr-vs-perf} illustrate the significance
of the teaching risk for the problem of teaching a learner under
worldview mismatch. To obtain these plots, we used a gridworld with
$N=20$, $n=2$; for each value $\ell \in [1, 100]$, we sampled five
random worldview matrices $A^L \in \R^{\ell \times 100}$, and let $L$
train a policy $\pi^L$ using the projection algorithm in
\cite{Abbeel-Ng--ICML04}, with the goal of matching the feature
expectations $\mu(\pi^T)$ corresponding to an optimal policy $\pi^T$
for a reward vector $\wopt$ that was sampled randomly in each round. Each
point in the plots corresponds to one such experiment and shows the
relative performance of $\pi_L$ after the training round vs.\ the
teaching risk of $L$'s worldview matrix $A^L$.

All plots in Figure \ref{fig:tr-vs-perf} show that the variance of the
learner's performance decreases as the teaching risk decreases. This
supports our interpretation of the teaching risk as a measure of the
potential gap between the performances of $\pi^L$ and $\pi^T$ when $L$
matches the feature expectations of $\pi^T$ in her worldview. The
plots also show that the bound for this gap provided in Theorem
\ref{thm:guarantee-bounded-risk} is overly conservative in general,
given that $L$'s performance is often high and has small variance even
if the teaching risk is relatively large.

The plots indicate that for larger $\gamma$ (i.e., less discounting),
it is easier for $L$ to achieve high performance even if the teaching
risk is large. This makes intuitive sense: If there is a lot of
discounting, it is important to reach high reward states quickly in
order to perform well, which necessitates being able to recognize
where these states are located, which in turn requires the teaching
risk to be small. If there is little discounting, it is sufficient to
know the location of \emph{some} maybe distant reward state, and hence
even a learner with a very deficient worldview (i.e., high teaching
risk) can do well in that case.

\begin{figure}[t]
    \centering
    \begin{subfigure}[t]{0.329\textwidth}
        \includegraphics[height=3.5cm]{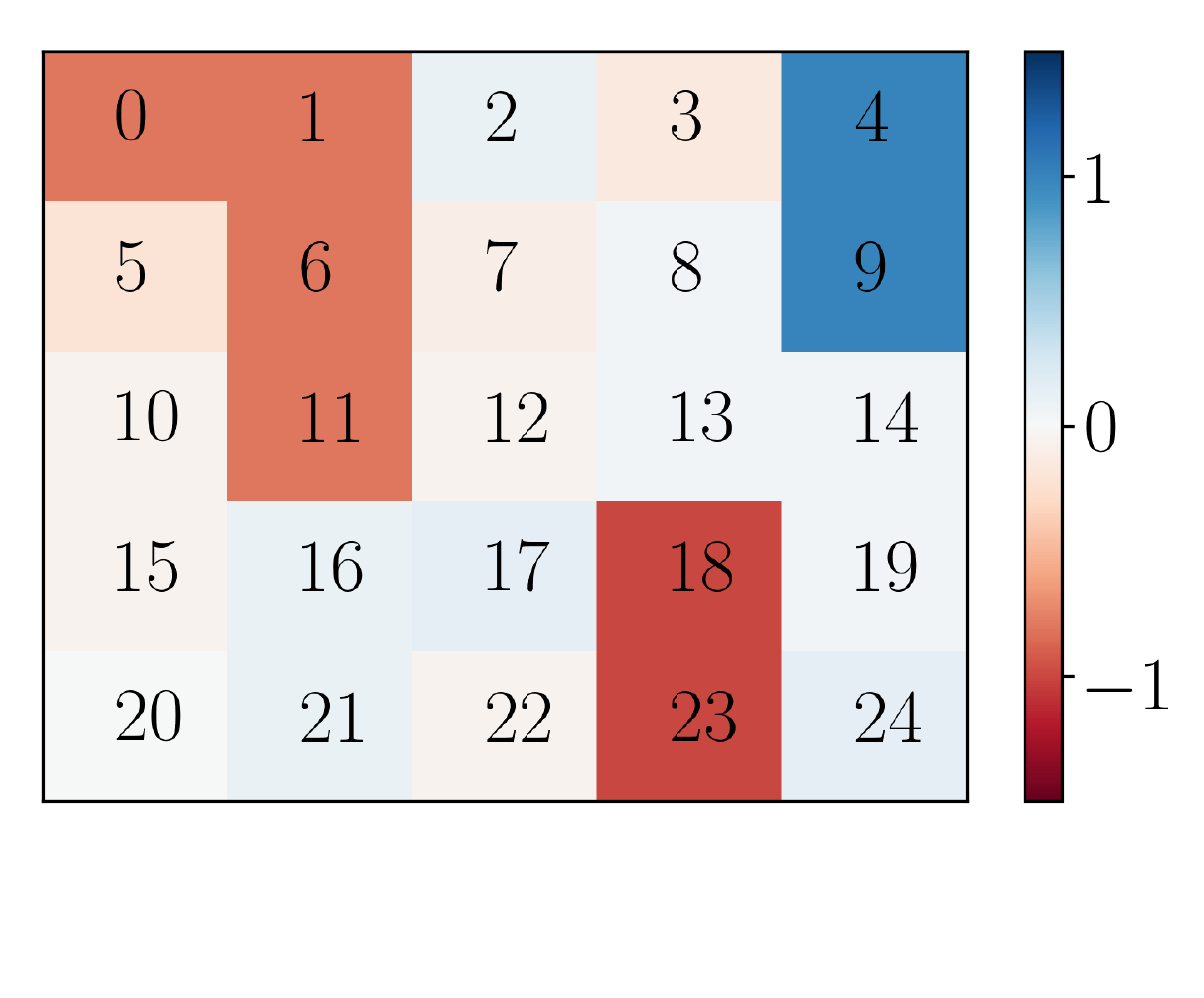}
        \subcaption{}
        \label{fig:gridworld-interpretable}
    \end{subfigure}
    \hfill
    \begin{subfigure}[t]{0.329\textwidth}
        \includegraphics[height=3.5cm]{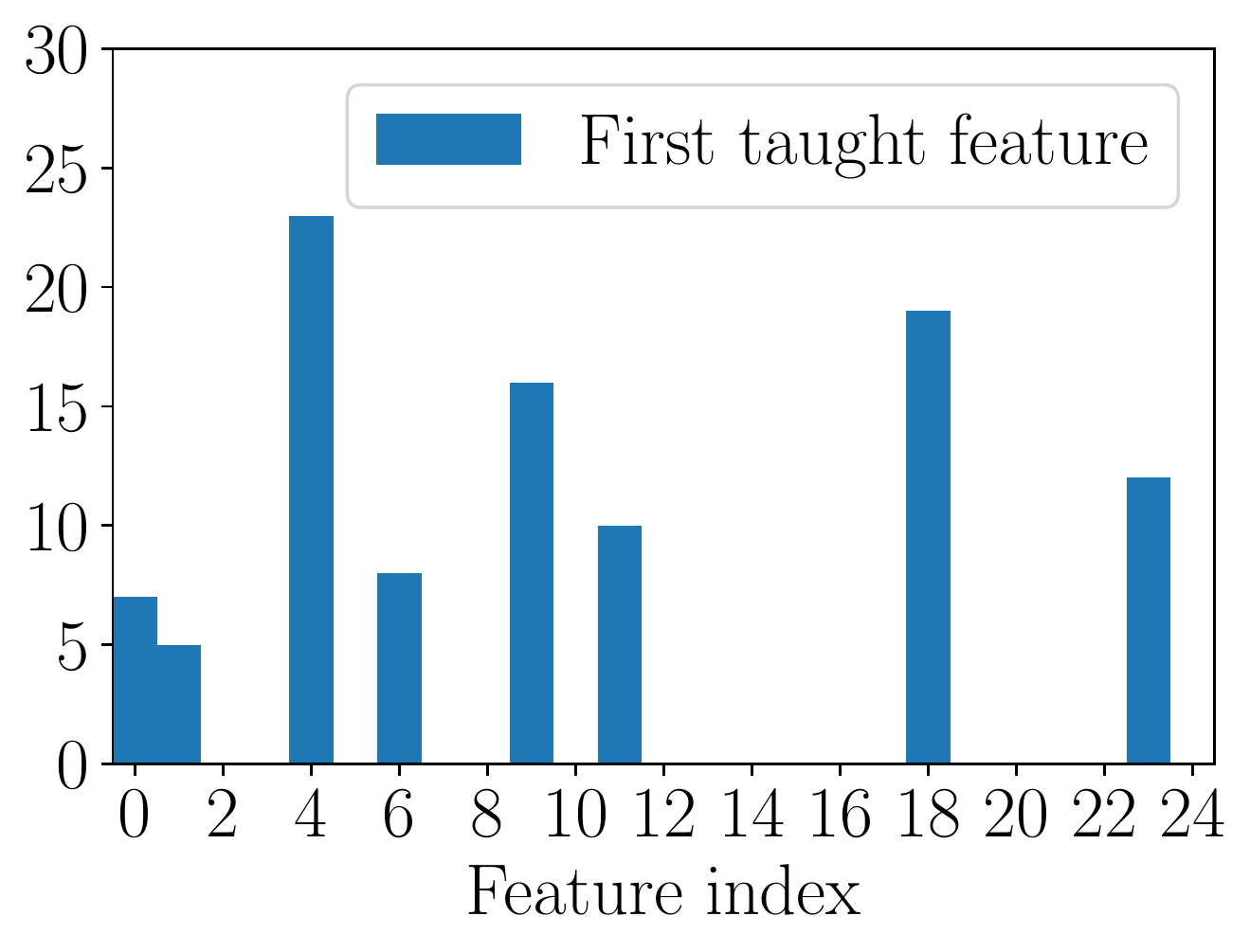}
        \subcaption{}
        \label{fig:histogram1}
    \end{subfigure}
    \begin{subfigure}[t]{0.329\textwidth}
        \includegraphics[height=3.5cm]{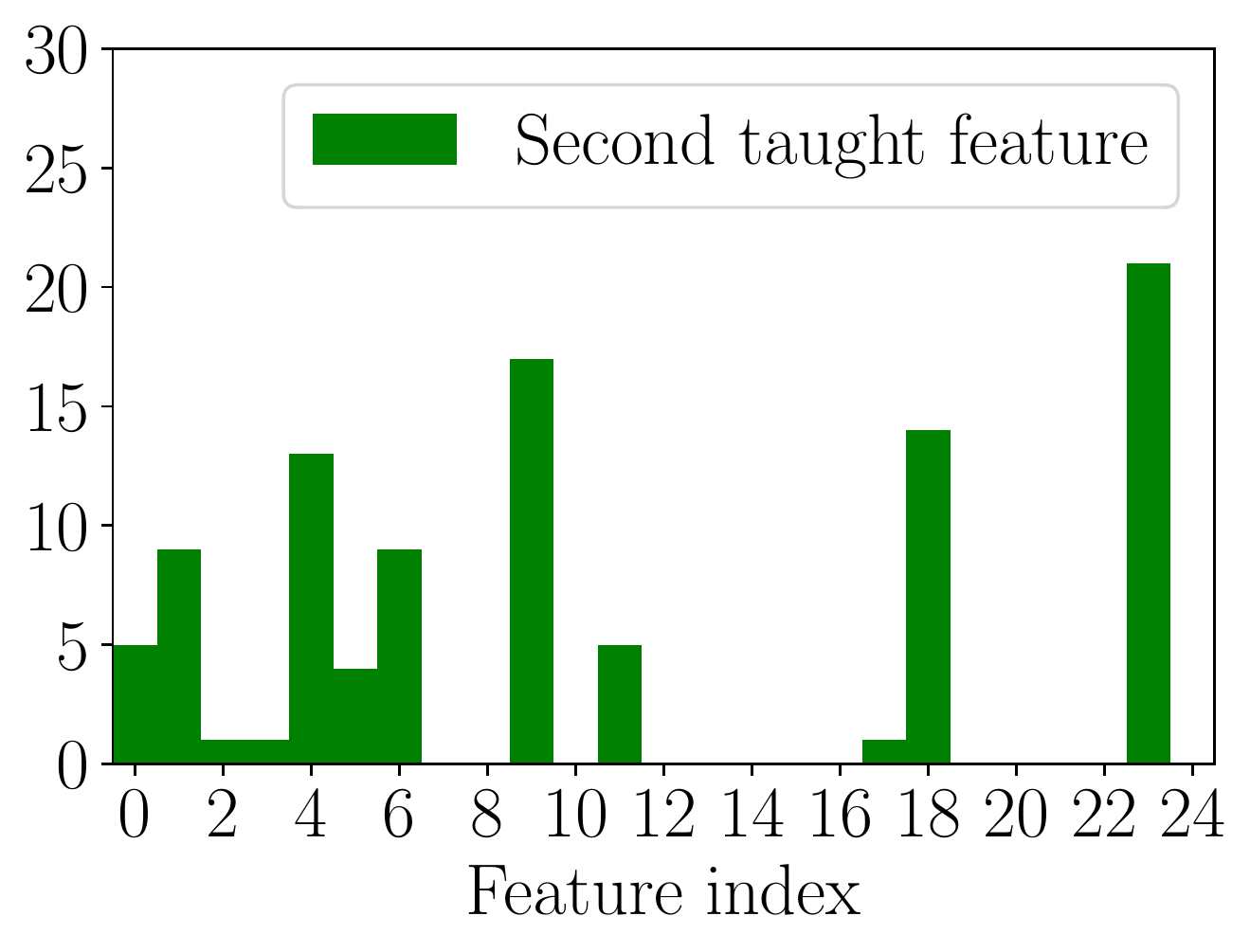}
        \subcaption{}
        \label{fig:histogram2}
    \end{subfigure}
    \caption{Gridworld with $N=10,
        n=2$. The colors in (\subref{fig:gridworld-interpretable}) indicate
        the reward of the corresponding macrocell, with blue meaning
        positive and red meaning negative reward. The numbers within
        each macrocell correspond to the feature index. The histograms
        in (\subref{fig:histogram1}) and (\subref{fig:histogram2})
        show how often, in a series of 100 experiments, each feature
        was selected as the first resp.\ second feature to be taught
        to a learner with a random 5-dimensional initial worldview.}
    \label{fig:histograms}
\end{figure}

\paragraph{Small gridworlds with high reward states and obstacles.}
We tested \textsc{TRGreedy} (Algorithm
\ref{algo:feature-and-demo-based-teaching}) on gridworlds such as the
one in Figure \ref{fig:gridworld-interpretable}, with a small number
of states with high positive rewards, some obstacle states with high
negative rewards, and all other states having rewards close to
zero. The histograms in Figures \ref{fig:histogram1} and
\ref{fig:histogram2} show how often each of the features was selected
by the algorithm as the first resp.\ second feature to be taught to
the learner in 100 experiments, in each of which the learner was
initialized with a random 5-dimensional worldview. In most cases, the
algorithm first selected the features corresponding to one of the high
reward cells 4 and 9 or to one of the obstacle cells 18 and 23, which
are clearly those that the learner must be most aware of in order to
achieve high performance.

\paragraph{Comparison of algorithms.} We compared the performance of
\textsc{TRGreedy} (Algorithm
\ref{algo:feature-and-demo-based-teaching}) to two variants of the
algorithm which are different only in how the features to be taught
in each round are selected: The first variant, \textsc{Random},
simply selects a random feature $f \in \cF$ from the set of all
teachable features. The second variant, \textsc{PerfGreedy}, 
greedily selects the feature that will lead to the best performance
in the next round among all $f \in \cF$ (computed by simulating the 
teaching process for each feature and evaluating the corresponding learner).

The plots in Figure \ref{fig:experimental-comparison} show, for each
of the three algorithms, the relative performance with respect to the
true reward function $s \mapsto \langle \wopt, \feat(s) \rangle$ that
the learner achieved after each round of feature teaching and training
a policy $\pi^L$, as well as the corresponding teaching risks and
runtimes, plotted over the number of features taught. The relative
performance of the learner's policy $\pi^L$ was computed as
$(R(\pi^L) - \min_{\pi}R(\pi))/(\max_{\pi} R(\pi) -
\min_{\pi}R(\pi))$.



We observed in all our experiments that \textsc{TRGreedy} performed
significantly better than \textsc{Random}. While the comparison
between \textsc{TRGreedy} and \textsc{PerfGreedy} was slightly in
favour of the latter, one should note that a teacher $T$ running
\textsc{PerfGreedy} must simulate a learning round of $L$ for all
features $f \in \mathcal F$ not yet taught, which presupposes that $T$
knows $L$'s learning algorithm, and which also leads to very high
runtime. If $T$ only knows that $L$ is able to match (her view of) the
feature expectations of $T$'s demonstrations and $T$ simulates $L$
using \emph{some} algorithm capable of this, there is no guarantee
that $L$ will perform as well as $T$'s simulated counterpart, as there
may be a large discrepancy between the true performances of two
policies which in $L$'s view have the same feature expectations. In
contrast, \textsc{TRGreedy} relies on much less information, namely
the kernel of $A^L$, and in particular is agnostic to the precise
learning algorithm that $L$ uses to approximate feature counts.

\begin{figure}[t]     
    \begin{subfigure}[]{0.329\textwidth}
        \includegraphics[width=\textwidth]{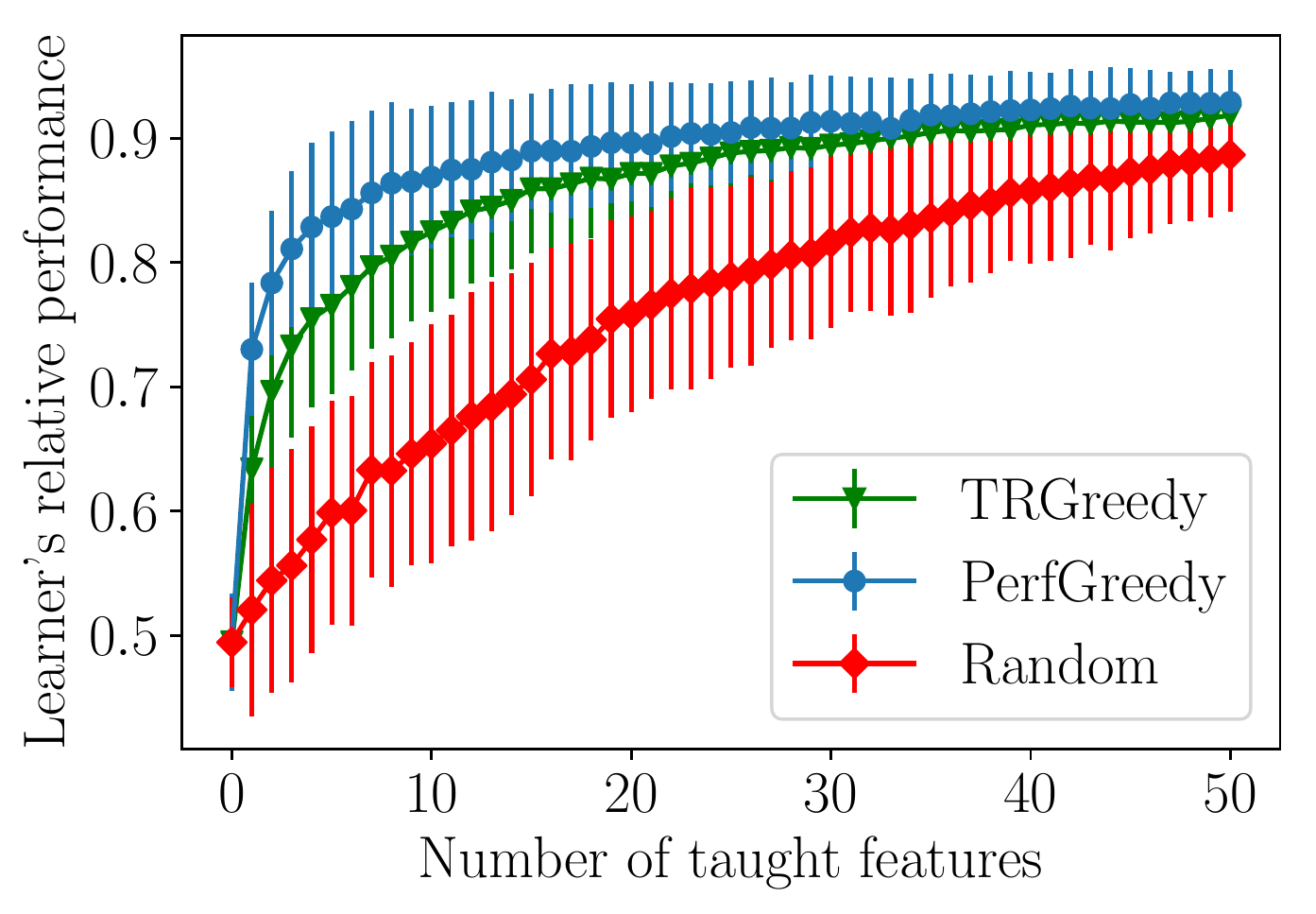}
        \subcaption{Relative performance}
        \label{fig:experimental-comparison-a}
    \end{subfigure}
    \begin{subfigure}[]{0.329\textwidth}
        \includegraphics[width=\textwidth]{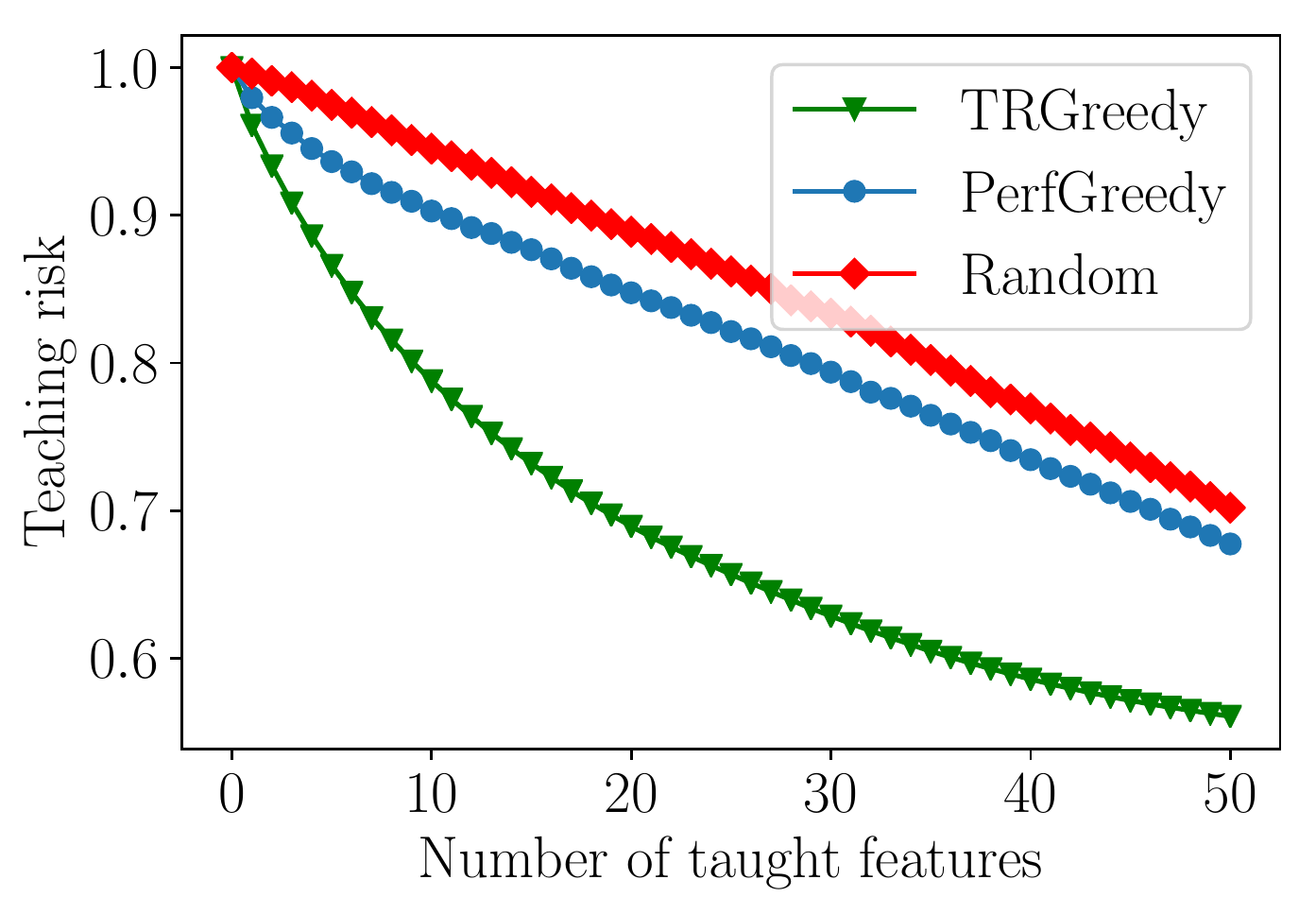}
        \subcaption{Teaching risk}
        \label{fig:experimental-comparison-b}
    \end{subfigure}
    \begin{subfigure}[]{0.329\textwidth}
        \includegraphics[width=\textwidth]{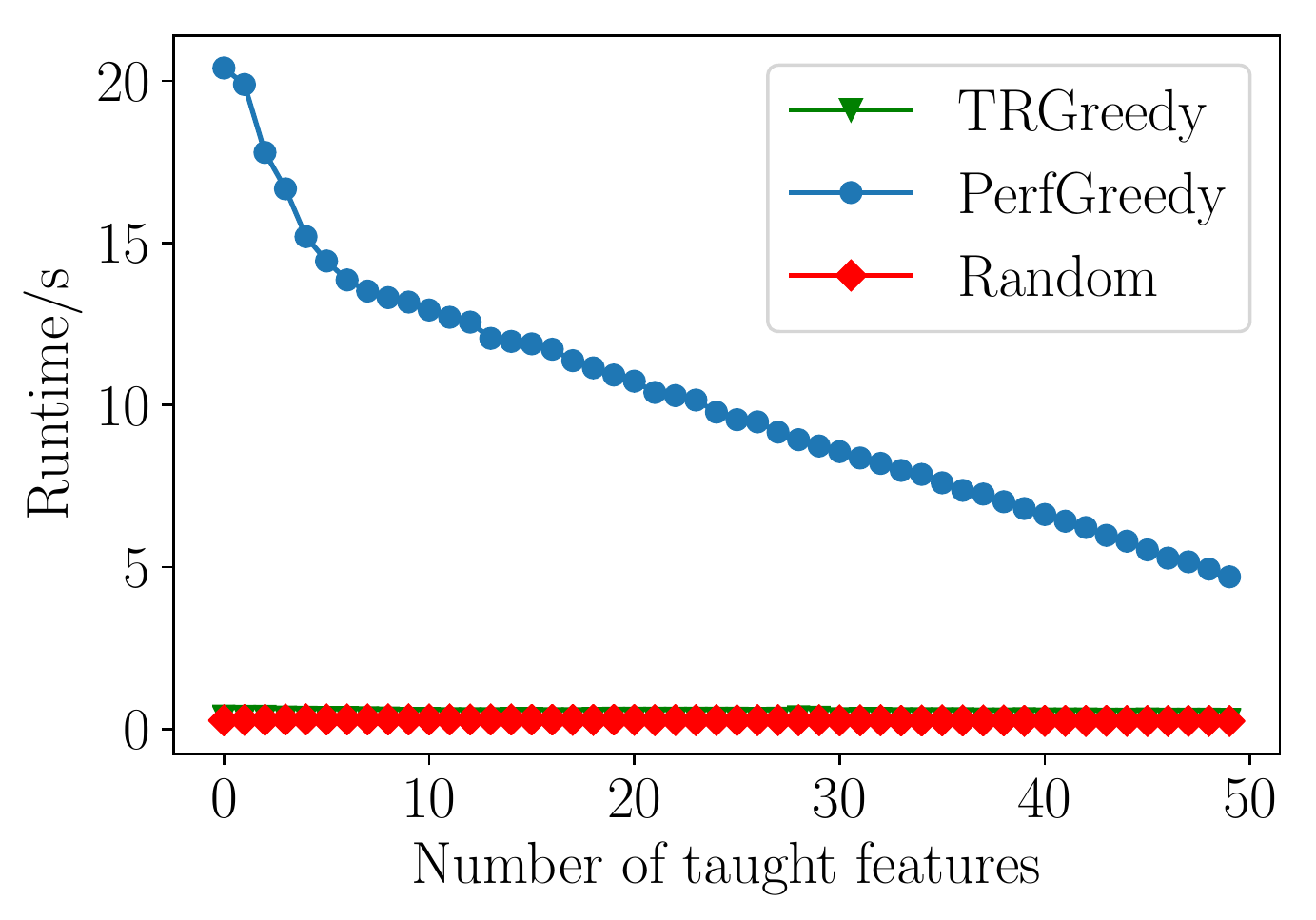}
        \subcaption{Runtime}
        \label{fig:experimental-comparison-c}
    \end{subfigure}
    \caption{Comparison of \textsc{TRGreedy} vs.\ \textsc{PerfGreedy}
        vs.\ \textsc{Random}. The plots show
        (\subref{fig:experimental-comparison-a}) the relative
        performance that the learner achieved after each round of
        feature teaching and training a policy,
        (\subref{fig:experimental-comparison-b}) the teaching risk
        after each such step, and
        (\subref{fig:experimental-comparison-c}) the runtime required
        to perform each step. We averaged over 100 experiments, in
        each of which a new random gridworld of size
        $(N, n) = (20, 2)$ and a new set $\cF$ of randomly selected features with
        $\vert \cF \vert = 70$ were sampled; the bars in the relative
        performance plot indicate the standard deviations. The
        discount factor used was $\gamma = 0.9$ in all cases.}
    \label{fig:experimental-comparison}
\end{figure}



\section{Conclusions and Outlook}\label{sec:conclusions}
We presented an approach to dealing with the problem of
\emph{worldview mismatch} in situations in which a learner attempts to
find a policy matching the feature counts of a teacher's
demonstrations. We introduced the \emph{teaching risk}, a quantity
that depends on the worldview of the learner and the true reward
function and which (1) measures the degree to which policies which are
optimal from the point of view of the learner can be suboptimal from
the point of view of the teacher, and (2) is an obstruction for truly
optimal policies to look optimal to the learner. We showed that under
the condition that the teaching risk is small, a learner matching
feature counts using, e.g., standard IRL-based methods is guaranteed to
learn a near-optimal policy from demonstrations of the teacher even
under worldview mismatch.

Based on these findings, we presented our teaching algorithm
\textsc{TRGreedy}, in which the teacher updates the learner's
worldview by teaching her features which are relevant for the true
reward function in a way that greedily minimizes the teaching risk,
and then provides her with demostrations based on which she learns a
policy using any suitable algorithm. We tested our algorithm in
gridworld settings and compared it to other ways of selecting features
to be taught. Experimentally, we found that \textsc{TRGreedy}
performed comparably to a variant which selected features based on
greedily maximizing performance, and consistently better than a
variant with randomly selected features.

We plan to investigate extensions of our ideas to nonlinear settings
and to test them in more complex environments in future work. We hope
that, ultimately, such extensions will be applicable in real-world
scenarios, for example in systems in which human expert knowledge is
represented as a reward function, and where the goal is to teach this
expert knowledge to human learners.


\bibliography{references}
\bibliographystyle{apalike}

\clearpage
\onecolumn
\appendix
{\allowdisplaybreaks
\section{Proof of Theorem \ref{thm:guarantee-bounded-risk}}

\begin{proof}[Proof of Theorem \ref{thm:guarantee-bounded-risk}]
    Denote by $\pr: \R^k \to \ker A^L$ the orthogonal projection onto
    $\ker A^L$ and let $v = \pr(\mu(\pi^L) - \mu(\pi^T))$. Note that
    we have $v + \mu(\pi^T) - \mu(\pi^L) \in (\ker A^L)^\perp$ and
    $\Vert v \Vert \leq \Vert \mu(\pi^T) - \mu(\pi^L) \Vert \leq
    \diam \mu(\Pi)$. It follows that
    \begin{align*}
      \Vert v + \mu(\pi^T) - \mu(\pi^L) \Vert &\leq \frac{1}{\sigma(A^L)} \Vert
                                                A^L(v + \mu(\pi^T) - \mu(\pi^L))
                                                \Vert \\
                                              &= \frac{1}{\sigma(A^L)} \Vert
                                                A^L(\mu(\pi^T) - \mu(\pi^L)) \Vert \\
                                              &< \frac{\varepsilon}{\sigma(A^L)},
    \end{align*}
    using the definition of $\sigma(A^L)$, the fact that $v \in \ker
    A^L$, and the assumption that $\Vert A^L(\mu(\pi^T) -
    \mu(\pi^L))\Vert < \varepsilon$. We then obtain
    \begin{align*}
      \vert \langle \wopt, \mu(\pi^T) \rangle - \langle \wopt, \mu(\pi^L)
      \rangle \vert &\leq \vert \langle \wopt, v + \mu(\pi^T) - \mu(\pi^L)
                      \rangle \vert + \vert \langle \wopt, v \rangle
                      \vert\\
                    &\leq \Vert v + \mu(\pi^T) - \mu(\pi^L) \Vert + \Vert v
                      \Vert \cdot \rho(A^L; \wopt)\\
                    &\leq \frac{\varepsilon}{\sigma(A^L)} +
                      \diam\mu(\Pi) \cdot \rho(A^L; \wopt)
    \end{align*}
    using the triangle inequality, the Cauchy-Schwarz inequality and
    the definition of $\rho(A^L; \wopt)$, and the estimates above.
\end{proof}

\section{Proof of Proposition \ref{prop:teaching-risk-suboptimality}}

\begin{figure}[h]
    \centering
    \includegraphics[scale=0.8]{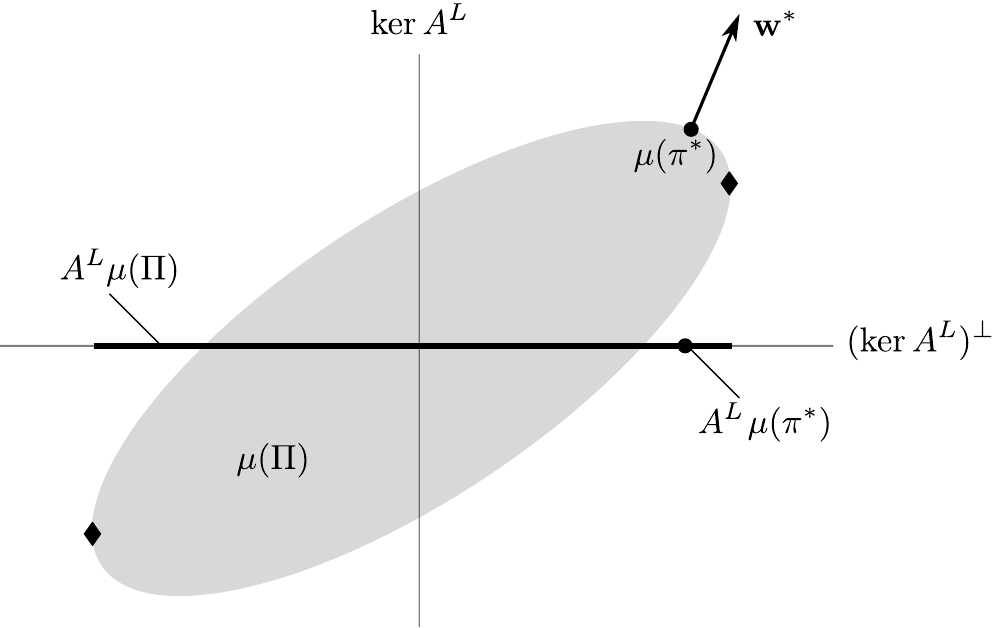}
    \caption{A situation in which $\rho(A^L; \wopt) > 0$: Here, $A^L$ is the
        projection on the horizontal axis. The points in
        $\partial \mu(\Pi)$ which $A^L$ maps to the boundary of
        $A^L\mu(\Pi)$, and which therefore appear optimal to $L$, are
        the two points marked by
        \raisebox{0.1em}{$\scriptstyle \blacklozenge$}, at which the
        normal vector to $\mu(\Pi)$ is contained in
        $(\ker A^L)^\perp$; these are precisely the points
        which are optimal for some $\wopt$ with $\rho(A^L; \wopt) = 0$
        (namely, $\wopt = (\pm 1, 0)$). All other points in
        $\partial \mu(\Pi)$ get mapped by $A^L$ to the interior of
        $A^L\mu(\Pi)$ and therefore appear suboptimal for any choice
        of reward function $L$ might consider.}
    \label{fig:suboptimality-pos-risk}
\end{figure}

\begin{proof}[Proof of Proposition
    \ref{prop:teaching-risk-suboptimality}]

    As mentioned in the main text, we assume that the set
    $\mu(\Pi) \subset \R^k$ is the closure of a bounded open set and
    has a smooth boundary $\partial \mu(\Pi)$.

    Finding the feature expectations $\mu(\pi^*)$ of a policy $\pi^*$
    which is optimal with respect to
    $s \mapsto \langle \wopt, \feat(s) \rangle$ is equivalent to
    maximizing the linear map $\mu \mapsto \langle \wopt, \mu \rangle$
    over $\mu(\Pi) \subset \R^k$, whose gradient is $\wopt \neq 0$. It
    follows that these feature expectations need to satisfy the
    following conditions:
    \begin{enumerate}
    \item $\mu(\pi^*)$ lies on the boundary $\partial \mu(\Pi)$,
    \item $\wopt$ is normal to $\partial \mu(\Pi)$ at $\mu(\pi^*)$.
    \end{enumerate}
    The second condition is equivalent to saying that the tangent
    space to $\mu(\Pi)$ at $\mu(\pi^*)$ is $\ker \wopt$.

    Assume now that $\rho(A^L; \wopt) > 0$. This is equivalent to saying
    that $\ker A^L \not \subseteq \ker \wopt$, i.e., to saying that $\ker
    A^L$ is not tangent to $\partial \mu(\Pi)$ at $\mu(\pi^*)$. That
    implies that there exist some $v \in \ker A^L$ such that $\mu(\pi^*) +
    v$ is contained in the interior of $\mu(\Pi)$, which means that a
    sufficiently small ball around $\mu(\pi^*) + v$ is contained in
    $\mu(\Pi)$. In particular, a small ball around $\mu(\pi^*) + v$ in the
    affine space $\mu(\pi^*) + v + (\ker A_L)^\perp$ is entirely contained
    in $\mu(\Pi)$. This implies that $A^L(\mu(\pi^*)) = A^L(\mu(\pi^* +
    v))$ is contained in the interior of $A^L \mu(\Pi)$, i.e., not in the
    boundary $\partial A^L(\mu(\pi^*))$. Therefore $\pi^*$ is suboptimal
    with respect to any choice of reward function $s \mapsto \langle \w,
    A^L \feat(s) \rangle$ with $\w \in \R^\ell$.
\end{proof}


\section{Proof of Theorem \ref{thm:guarantee-bounded-suboptimality}}

\begin{proof}[Proof of Theorem \ref{thm:guarantee-bounded-suboptimality}]
    The assumption that $\pi^*$ is optimal for the reward
    function $s \mapsto \langle \wopt, \feat(s) \rangle$ implies that
    $\langle \wopt, \mu - \mu(\pi^*) \rangle \leq 0$ for all
    $\mu \in \mu(\Pi)$. By decomposing $\wopt$ as
    $\wopt = \pr(\wopt) + \pr^\perp(\wopt)$, where
    $\pr: \R^k \to \ker A^L$ denotes the orthogonal projection onto
    $\ker A^L$ and $\pr^\perp: \R^k \to (\ker A^L)^\perp$ the
    orthogonal projection onto $(\ker A^L)^\perp$, we obtain
    \begin{equation}
        \label{eq:est1}
        \langle \pr(\wopt), \mu - \mu(\pi^*) \rangle + \langle
        \pr^\perp(\wopt), \mu - \mu(\pi^*) \rangle \leq 0.
    \end{equation}
    The first summand can be bounded as follows:
    \begin{align}
      \label{eq:est2}
      \begin{split}
          \langle \pr(\wopt), \mu - \mu(\pi^*) \rangle &\geq - \Vert
          \mu - \mu(\pi^*) \Vert \cdot \Vert
          \pr(\wopt) \Vert \\
          &= -\Vert \mu - \mu(\pi^*) \Vert \cdot \rho(A^L; \wopt)\\
          &\geq -\diam \mu(\Pi) \cdot \rho(A^L; \wopt),
      \end{split}
    \end{align}
    using the Cauchy-Schwarz inequality and the fact that
    $\Vert \pr(\wopt) \Vert = \rho(A^L; \wopt)$. By combining
    estimates \eqref{eq:est1} and \eqref{eq:est2}, we obtain
    \begin{equation}
        \label{eq:ineq}
        \langle \pr^\perp(\wopt), \mu - \mu(\pi^*) \rangle \leq \diam
        \mu(\Pi) \cdot \rho(A^L; \wopt).
    \end{equation}

    Denote now by $(A^L)^+$ the Moore-Penrose pseudoinverse of $A^L$,
    and by $X: =( (A^L)^+)^T$ its transpose. We have
    \begin{align}
      \label{eq:ident1}
      \begin{split}
          \langle \pr^\perp(\wopt), \mu - \mu(\pi^*) \rangle &=
          \langle \wopt, \pr^\perp(\mu -
          \mu(\pi^*) \rangle\\
          &= \langle \wopt, (A^L)^+A^L\pr^\perp (\mu - \mu(\pi^*)
          \rangle\\
          &= \langle X \wopt, A^L\pr^\perp(\mu
          - \mu(\pi^*)\rangle\\
          &= \langle X\wopt, A^L (\mu - \mu(\pi^*)\rangle,
      \end{split}
    \end{align}
    where the second equality uses the fact that the restriction of
    $(A^L)^+ A^L$ to $(\ker A^L)^\perp$ is the identity (in fact,
    $(A^L)^+A^L = \pr^\perp$, a general property of Moore-Penrose
    pseudoinverses). Setting
    $\wopt_L := \frac{1}{\Vert X \wopt \Vert}{X\wopt}$ and combining
    inequality \eqref{eq:ineq} with \eqref{eq:ident1}, we obtain
    \begin{equation}
        \label{eq:est3}
        \left \langle \wopt_L, A^L
            (\mu - \mu(\pi^*))\right\rangle \leq \frac{\diam \mu(\Pi) \cdot
            \rho(A^L; \wopt)}{\Vert X \wopt \Vert}.
    \end{equation}

    We now estimate the term $\Vert X \wopt \Vert$:
    \begin{align}
      \label{eq:est4}
      \begin{split}
          \Vert X \wopt \Vert &= \max_{v \in \R^\ell, \Vert v \Vert =
              1}
          \langle X\wopt, v \rangle\\
          &= \max_{v \in \R^\ell, \Vert v \Vert = 1}
          \langle \wopt, (A^L)^+ v \rangle\\
          &\geq \frac{\langle \wopt, (A^L)^+ A^L \pr^\perp(\wopt)\rangle}{
              \Vert A^L
              \pr^\perp(\wopt) \Vert} \\
          &= \frac{\Vert \pr^\perp(\wopt) \Vert^2}{\Vert
              A^L \pr^\perp(\wopt) \Vert}\\
          &= \frac{\Vert \pr^\perp(\wopt) \Vert}{\Vert A^L
              \frac{\pr^\perp(\wopt)}{\Vert
                  \pr^\perp(\wopt)\Vert}\Vert}\\
          &\geq \frac{\Vert \pr^\perp(\wopt) \Vert}{\Vert A^L\Vert}.
      \end{split}
    \end{align}
    Since
    $\Vert \pr^\perp(\wopt) \Vert = \sqrt{\Vert \wopt \Vert^2 - \Vert
        \pr(\wopt) \Vert^2} = \sqrt{1 - \rho(A^L; \wopt)^2}$,
    combining \eqref{eq:est3} and \eqref{eq:est4} yields
    \begin{equation*}
        \langle \wopt_L, A^L(\mu - \mu(\pi^*)) \rangle \leq \frac{\diam
            \mu(\Pi) \cdot \Vert A^L\Vert \cdot \rho(A^L; \wopt)}{\sqrt{1 - \rho(A^L; \wopt)^2}}
    \end{equation*}
    This holds for all $\mu \in \mu(\Pi)$, and hence we can maximize
    over $\mu$ to obtain the statement claimed in Theorem
    \ref{thm:guarantee-bounded-suboptimality}.
\end{proof}

}
\end{document}